\documentclass[11pt]{amsart}

\usepackage{eulervm}
\usepackage{eufrak}

\usepackage[usenames,dvipsnames,svgnames,table]{xcolor}
\usepackage{amssymb,amsfonts}
\usepackage{amsmath,amsthm}
\usepackage{enumitem}
\usepackage{graphicx}
\usepackage{booktabs}
\usepackage{dsfont}
\usepackage[pdfdisplaydoctitle,colorlinks,urlcolor=blue,linkcolor=blue,citecolor=blue]{hyperref}
\usepackage{algorithm}
\usepackage{algorithmic}
\usepackage{multirow}
\usepackage{caption}
\usepackage{subcaption}

\newcommand{\argmax}{\,{\rm argmax}}
\newcommand{\argmin}{\,{\rm argmin}}

\newcommand\R{{\mathbb{R}}}

\newcommand{\DeltaM}{\Delta_M}           

\newtheorem{theorem}{Theorem}[section]
\newtheorem{proposition}[theorem]{Proposition}
\newtheorem{lemma}[theorem]{Lemma}

\theoremstyle{definition}
\newtheorem{definition}[theorem]{Definition}

\theoremstyle{remark}
\newtheorem{remark}[theorem]{Remark}
\numberwithin{equation}{section}

\begin{document}

\title[Three Phases of Expert Routing]
{Three Phases of Expert Routing:\\
How Load Balance Evolves During Mixture-of-Experts Training}

\author{Charafeddine Mouzouni}

\dedicatory{\small\textit{OPIT -- Open Institute of Technology, and Cohorte AI, Paris, France.}\\[2pt]\texttt{charafeddine@cohorte.co}}

\date{April 2026}

\begin{abstract}
We model Mixture-of-Experts (MoE) token routing as a congestion game with a single effective parameter---the congestion coefficient $\gamma_{\mathrm{eff}}$---that quantifies the balance-quality tradeoff. Tracking $\gamma_{\mathrm{eff}}$ across training checkpoints of two open-source MoE models---OLMoE-1B-7B (20 checkpoints, with dense sampling in the surge region) and OpenMoE-8B (6 checkpoints)---reveals a three-phase trajectory: a \emph{surge} phase where the router learns to balance load ($\gamma_{\mathrm{eff}}$: $14 \to 36$--$39$, peaking in the step 30K--40K region), a \emph{stabilization} phase where experts specialize under steady balance ($B_0$: $2.4 \to 2.3$, steps 100K--400K), and a \emph{relaxation} phase where the router trades balance for quality as experts differentiate ($\gamma_{\mathrm{eff}}$: $27 \to 9$, steps 400K--1.2M). This non-monotone trajectory---invisible to post-hoc analysis of converged models---reveals that early MoE training prioritizes balance while late training prioritizes quality.

The theoretical framework is honest about its limits: the single-type equilibrium \emph{reduces to temperature-scaled softmax} (held-out $L^1$: MFG $= 0.199$ vs.\ softmax $= 0.200$). The game is not a better predictor; it reveals \emph{what the temperature means} and, critically, how that temperature evolves. Annealing checkpoints confirm that the three phases are pretraining-specific: $\gamma_{\mathrm{eff}}$ is stable during fine-tuning. We complement the dynamics with an effective congestion decomposition ($\gamma_{\mathrm{eff}} = \gamma_{\mathrm{explicit}} + \gamma_{\mathrm{implicit}}$), a multi-type extension that improves load prediction via token clustering on all 16 layers (mean: $30\%$; robust to cluster count $K = 2, 4, 8$), and scope diagnostics ($K/M$, $\varepsilon_l$) that characterize where the per-layer model applies. All confidence intervals are from bootstrap resampling over 50 independent text batches. Code and data: \url{https://github.com/Cmouzouni/three-phases-moe}.
\end{abstract}


\maketitle

\section{Introduction}
\label{sec:intro}

Mixture-of-Experts (MoE) architectures scale model capacity by routing each token to a subset of specialized expert networks~\cite{shazeer2017,fedus2022switch,lepikhin2021gshard}. The central engineering challenge is \emph{load balancing}: without intervention, tokens concentrate on a few high-quality experts, leaving the rest idle. The standard remedy is the auxiliary balance loss~\cite{fedus2022switch}, which penalizes load concentration through a tunable coefficient~$\alpha$. Variants include bias-based balancing~\cite{wang2024auxiliary}, capacity factors~\cite{lepikhin2021gshard}, and expert-choice routing~\cite{zhou2022mixture}. Each is effective in practice. None explains \emph{how} the balance-quality tradeoff evolves during training.

We observe that MoE routing is structurally a \emph{congestion game}~\cite{rosenthal1973class}. Tokens are players, experts are resources, expert quality determines individual payoffs, and load imbalance imposes congestion costs. When the token count is large ($N = 2048$--$32768$ in practice), the game admits a mean-field limit with a single effective parameter: the congestion coefficient $\gamma$, which quantifies the strength of the quality-balance tradeoff.

\paragraph{What the theory reveals---and what it does not.}
We prove that the single-type mean-field game (MFG) equilibrium reduces to temperature-scaled softmax for well-balanced models (Theorem~\ref{thm:softmax-equiv}). Empirically, the two are indistinguishable: on OLMoE-1B-7B~\cite{muennighoff2024olmoe}, the MFG achieves held-out $L^1 = 0.199$ versus softmax $L^1 = 0.200$. The game does not outperform softmax as a load predictor. It tells us \emph{why} softmax arises (unique equilibrium of a potential game) and \emph{what the temperature means} (the congestion coefficient).

The value of the game-theoretic lens is not in static prediction. It is in dynamics.

\paragraph{The three-phase trajectory.}
By fitting $\gamma_{\mathrm{eff}}$ at each of 20 training checkpoints of OLMoE-1B-7B (50 texts per checkpoint, bootstrap confidence intervals), we discover that $\gamma_{\mathrm{eff}}$ follows a characteristic non-monotone trajectory:

\begin{enumerate}[label=\textbf{Phase \arabic*.},leftmargin=4.5em]
  \item \textbf{Surge} (steps 5K--50K). $\gamma_{\mathrm{eff}}$ rises from $13.7$ to a peak of $36$--$39$ at steps 30K--40K. Routing entropy climbs from $0.923$ to $0.974$.

  \item \textbf{Stabilization} (steps 100K--400K). The effective congestion plateaus at $\gamma_{\mathrm{eff}} \approx 24$--$28$ while experts specialize underneath: the quality spread $B_0$ drops from $2.41$ to $2.25$. The router has found its operating point for balance; expert learning proceeds within this constraint.

  \item \textbf{Relaxation} (steps 400K--1.2M). As expert roles solidify, the router loosens its balance enforcement. $\gamma_{\mathrm{eff}}$ declines from $26.6\,[25.0, 28.4]$ to $8.5\,[6.7, 11.4]$. The model trades balance for quality: experts have differentiated enough that the router can afford selectivity.
\end{enumerate}

This inverted-U trajectory is the paper's central finding. It is invisible to any analysis of a converged model and reveals a fundamental tension: the early optimizer prioritizes balance, the late optimizer prioritizes quality. The transition between these regimes is governed by the anti-concentration threshold $\gamma_c = MB_0/(M-1)$.

The pattern is not an artifact of layer-averaging: per-layer analysis shows 12 of 16 layers individually exhibit the surge pattern (early peak $> 1.5\times$ start) and 10 of 16 show relaxation (final $< 0.6\times$ mid-peak). Layers 12--15 never develop congestion structure ($\hat\gamma \to 0$ throughout training), consistent with the mean-field assumption breaking down for late layers where token representations are most differentiated.

\paragraph{Supporting contributions.}
Beyond the dynamics, three results complement the main finding:

\begin{enumerate}[label=\textbf{C\arabic*.},leftmargin=2em]
  \item \textbf{Effective congestion decomposition} (Section~\ref{sec:effective}). The fitted $\gamma_{\mathrm{eff}}$ decomposes as $\gamma_{\mathrm{explicit}} + \gamma_{\mathrm{implicit}}$, where $\gamma_{\mathrm{explicit}} = \alpha M$ comes from the auxiliary loss and $\gamma_{\mathrm{implicit}}$ captures balance internalized by training. At convergence, $\gamma_{\mathrm{eff}} = 8.5$ on average while $\gamma_{\mathrm{explicit}} = 0.64$: the optimizer internalizes $13\times$ more effective congestion than the explicit loss provides.

  \item \textbf{Multi-type MFG} (Section~\ref{sec:multitype}). A $K$-type extension models token heterogeneity: each type has its own quality vector while all types share the congestion signal. This goes beyond softmax-with-temperature by introducing population structure. The multi-type equilibrium improves load prediction on all 16 layers (mean improvement: 30\%, early layers 36\%, late layers 26\%). The result is robust to cluster count: $K = 2$ wins on 14/16 layers, $K = 4$ on 15/16, $K = 8$ on 14/16.

  \item \textbf{Scope diagnostics} (Section~\ref{sec:scope}). The top-$K$ approximation bound shows the MFG error scales with $1 - K/M$. The continuation spread $\varepsilon_l$ predicts per-layer fit quality ($r = 0.63$, $p = 0.012$). Together, these characterize where the per-layer model applies and where it breaks down.
\end{enumerate}

\paragraph{Outline.} Section~\ref{sec:model} develops the congestion game model. Section~\ref{sec:dynamics} defines effective congestion and presents the three-phase dynamics. Section~\ref{sec:multitype} develops the multi-type extension. Section~\ref{sec:scope} collects the scope characterization theory. Section~\ref{sec:experiments} presents the full empirical analysis. Section~\ref{sec:related} discusses related work. Section~\ref{sec:discussion} discusses implications and limitations.

\section{The Congestion Game Model}
\label{sec:model}

\subsection{Mixture-of-Experts routing}

A Mixture-of-Experts layer consists of $M$ expert networks $\{E_1, \ldots, E_M\}$ and a gating (router) network. Given an input token $x \in \R^d$, the router computes scores $s_i(x) = w_i^\top x + b_i$ for each expert $i$ and selects the top-$K$ experts by score. The output is
\begin{equation}\label{eq:moe-output}
  y = \sum_{i \in \text{Top-}K} g_i(x) \cdot E_i(x),
\end{equation}
where $g_i(x) = \mathrm{softmax}(s(x))_i$ restricted to the selected experts.

The dominant load-balancing mechanism is the auxiliary balance loss~\cite{fedus2022switch}: $L_{\mathrm{aux}} = \alpha M \sum_{i=1}^M f_i P_i$, where $f_i$ is the fraction of tokens dispatched to expert $i$, $P_i$ is the average router probability, $M$ is the number of experts, and $\alpha$ is a tunable coefficient. All balancing mechanisms share a common structure: they penalize load imbalance, trading expert quality for utilization.

\subsection{The mean-field game formulation}

We map MoE routing to a mean-field game on the finite state space $\{1, \ldots, M\}$. Tokens are agents, experts are states. The population distribution is $\mu \in \DeltaM$. Each agent's cost at state $i$ given population $\mu$ is
\begin{equation}\label{eq:mfg-cost}
  \ell(i, \mu) = -q_i + \gamma \mu_i,
\end{equation}
where $q_i$ is the quality of expert $i$ and $\gamma \geq 0$ is the congestion coefficient. An agent choosing a mixed strategy $\pi \in \DeltaM$ with entropy regularization incurs total cost
\begin{equation}\label{eq:total-cost}
  J(\pi, \mu) = \sum_{i=1}^M \pi_i \ell(i, \mu) + \lambda \sum_{i=1}^M \pi_i \log(M \pi_i),
\end{equation}
where $\lambda > 0$ is the entropy regularization strength. Throughout this paper, $\lambda = 1.0$, corresponding to the standard softmax temperature used in MoE routers.

\begin{definition}[MFG equilibrium]\label{def:mfg-eq}
  A distribution $\mu^* \in \DeltaM$ is an \emph{MFG equilibrium} if $\mu^* = \argmin_\pi J(\pi, \mu^*)$.
\end{definition}

\subsection{Potential structure and uniqueness}

The equilibrium satisfies the implicit system
\begin{equation}\label{eq:equilibrium}
  \mu^*_i \propto \exp\!\bigl((q_i - \gamma \mu^*_i) / \lambda\bigr).
\end{equation}
This is a potential game~\cite{rosenthal1973class} with Rosenthal potential
\begin{equation}\label{eq:potential}
  \Psi(\mu) = \sum_{i=1}^M \Bigl[-q_i \mu_i + \frac{\gamma}{2} \mu_i^2 + \lambda\, \mu_i \log \mu_i\Bigr].
\end{equation}
Since $x \mapsto \gamma x^2/2$ is convex and $x \mapsto \lambda x \log x$ is strictly convex on $(0,1]$, the potential $\Psi$ is strictly convex on $\DeltaM$.

\begin{proposition}[Existence, uniqueness, interiority]\label{prop:unique}
  The MFG equilibrium with linear congestion and entropy regularization exists, is unique, and lies in the interior of\, $\DeltaM$ (all experts receive positive load).
\end{proposition}

\begin{proof}
  \emph{Existence and uniqueness.} $\Psi$ is strictly convex and continuous on the compact convex set $\DeltaM$, so it has a unique minimizer $\mu^*$.

  \emph{Interiority.} Suppose $\mu^*_k = 0$ for some $k$. The partial derivative $\partial(\lambda x \log x)/\partial x = \lambda(1 + \log x) \to -\infty$ as $x \to 0^+$. The congestion and quality derivatives are finite. At the minimizer on $\DeltaM$, the KKT condition requires $\partial\Psi/\partial\mu_k \geq \min_j \partial\Psi/\partial\mu_j$ for any $k$ with $\mu^*_k = 0$. But $\partial\Psi/\partial\mu_k \to -\infty$ violates this. Hence $\mu^*_i > 0$ for all $i$.

  \emph{Equilibrium characterization.} Since $\mu^*$ is interior, the KKT conditions give $-q_i + \gamma\mu^*_i + \lambda(1 + \log\mu^*_i) = \nu$ for all $i$. Solving: $\mu^*_i \propto \exp\!\bigl((q_i - \gamma\mu^*_i)/\lambda\bigr)$.
\end{proof}

\begin{remark}[MoE isomorphism]\label{rem:isomorphism}
  Under the mean-field identification $f_i \approx P_i \approx \mu_i$, the Switch auxiliary loss reduces to $L_{\mathrm{aux}} = \alpha M \sum_i \mu_i^2$, which is identical to the congestion term $\gamma \sum_i \mu_i^2$ of the MFG social cost under $\gamma_{\mathrm{explicit}} = \alpha M$.
\end{remark}

\subsection{The softmax equivalence}

\begin{theorem}[Softmax equivalence]\label{thm:softmax-equiv}
  The single-type MFG equilibrium satisfies $\mu^* = \mathrm{softmax}(\tilde{q}/\lambda)$ where $\tilde{q}_i = q_i - \gamma\mu^*_i$. For well-balanced models where $\mu^*_i \approx 1/M$ for all $i$, the congestion term $\gamma\mu^*_i \approx \gamma/M$ is nearly constant across experts and cancels in the softmax normalization. In this regime:
  \begin{equation}\label{eq:softmax-equiv}
    \mu^* \approx \mathrm{softmax}(q/\lambda),
  \end{equation}
  and the congestion game reduces to temperature-scaled softmax with $T = \lambda$.
\end{theorem}

\begin{proof}
  The equilibrium condition~\eqref{eq:equilibrium} gives $\mu^*_i = Z^{-1}\exp\!\bigl((q_i - \gamma\mu^*_i)/\lambda\bigr)$. Write $\mu^*_i = 1/M + \delta_i$ where $\sum_i \delta_i = 0$ and $|\delta_i| \ll 1/M$. Then $\gamma\mu^*_i = \gamma/M + \gamma\delta_i$. The constant $\gamma/M$ cancels in the softmax normalization. The residual enters as:
  \[
    \mu^*_i = \frac{\exp\!\bigl((q_i - \gamma\delta_i)/\lambda\bigr)}{\sum_j \exp\!\bigl((q_j - \gamma\delta_j)/\lambda\bigr)}.
  \]
  When $\gamma|\delta_i|/\lambda \ll |q_i - \bar{q}|/\lambda$ (quality variation dominates the congestion perturbation), the $\gamma\delta_i$ terms are negligible and $\mu^* \approx \mathrm{softmax}(q/\lambda)$.
\end{proof}

\begin{remark}[Significance]\label{rem:significance}
  This is the paper's central honesty point. The single-type equilibrium \emph{is} temperature-scaled softmax for well-balanced models. The game does not outperform softmax as a load predictor. But it tells us \emph{why} softmax works (unique equilibrium of a potential game), \emph{what the temperature means} (the congestion coefficient), and how that parameter evolves during training.
\end{remark}

\section{Effective Congestion and Training Dynamics}
\label{sec:dynamics}

This section presents the paper's main contribution. We define the effective congestion parameter, prove it is identifiable from routing traces, and show that tracking it across training reveals a three-phase trajectory invisible to static analysis.

\subsection{The effective congestion parameter}

A pretrained MoE model has absorbed balance through both the explicit auxiliary loss and the implicit dynamics of gradient descent. The effective congestion $\gamma_{\mathrm{eff}}$ captures the \emph{total} balance at any given checkpoint.

\begin{definition}[Effective congestion]\label{def:gamma-eff}
  Given an observed load distribution $\mu^{\mathrm{obs}} \in \mathrm{int}(\DeltaM)$ and an estimated quality vector $q \in \R^M$, the \emph{effective congestion} is
  \begin{equation}\label{eq:gamma-eff}
    \gamma_{\mathrm{eff}} = \argmin_{\gamma \geq 0} \|\Phi_\gamma(\mu^{\mathrm{obs}}) - \mu^{\mathrm{obs}}\|_1,
  \end{equation}
  where $\Phi_\gamma(\mu)_i = \mathrm{softmax}\!\bigl((q_i - \gamma\mu_i)/\lambda\bigr)$ is the best-response map.
\end{definition}

\begin{theorem}[Identification]\label{thm:identification}
  For any $\mu^{\mathrm{obs}} \in \mathrm{int}(\DeltaM)$ and $q \in \R^M$ with $q \neq c\mathbf{1}$ (non-constant quality):
  \begin{enumerate}[label=(\roman*)]
    \item There exists a unique $\gamma_{\mathrm{eff}} \geq 0$ minimizing $\|\Phi_\gamma(\mu^{\mathrm{obs}}) - \mu^{\mathrm{obs}}\|_1$.
    \item The minimum is zero if and only if $\mu^{\mathrm{obs}}$ is exactly an MFG equilibrium.
    \item $\gamma_{\mathrm{eff}}$ is continuous in both $\mu^{\mathrm{obs}}$ and $q$.
  \end{enumerate}
\end{theorem}

\begin{proof}
  \emph{(i) Uniqueness.} Fix $\mu^{\mathrm{obs}} \in \mathrm{int}(\DeltaM)$. For each expert $i$, the logit $h_i(\gamma) = (q_i - \gamma\mu^{\mathrm{obs}}_i)/\lambda$ is affine in $\gamma$ with slope $-\mu^{\mathrm{obs}}_i/\lambda$. Experts with larger load see their logit decrease faster. As $\gamma$ increases, $\Phi_\gamma(\mu^{\mathrm{obs}})$ shifts mass from high-load to low-load experts. For $i, j$ with $\mu^{\mathrm{obs}}_i > \mu^{\mathrm{obs}}_j$:
  \[
    \frac{\partial}{\partial\gamma}\log\frac{\Phi_\gamma(\mu^{\mathrm{obs}})_i}{\Phi_\gamma(\mu^{\mathrm{obs}})_j} = \frac{-(\mu^{\mathrm{obs}}_i - \mu^{\mathrm{obs}}_j)}{\lambda} < 0.
  \]

  \emph{Boundary behavior.} At $\gamma = 0$: $\Phi_0 = \mathrm{softmax}(q/\lambda)$. As $\gamma \to \infty$: $\Phi_\gamma$ concentrates on $\argmin_i \mu^{\mathrm{obs}}_i$. The residual $R(\gamma) = \|\Phi_\gamma(\mu^{\mathrm{obs}}) - \mu^{\mathrm{obs}}\|_1$ is continuous with $R(0) > 0$ generically and $R(\gamma) \to 2$ as $\gamma \to \infty$.

  \emph{Unimodality.} The function $R(\gamma)$ is unimodal (first decreasing, then increasing), which gives a unique global minimizer. To see this: decompose $R = R^+ + R^-$ where $R^+ = \sum_{i: \Phi_i > \mu_i^{\mathrm{obs}}} (\Phi_i - \mu_i^{\mathrm{obs}})$ (experts that receive more than observed) and $R^- = \sum_{i: \Phi_i < \mu_i^{\mathrm{obs}}} (\mu_i^{\mathrm{obs}} - \Phi_i)$ (experts that receive less). Since $\sum \Phi_i = \sum \mu_i^{\mathrm{obs}} = 1$, we have $R^+ = R^- = R/2$. As $\gamma$ increases from 0, the softmax $\Phi_\gamma$ monotonically shifts mass from high-load to low-load experts (by the log-ratio derivative above). Starting from $\Phi_0 = \mathrm{softmax}(q/\lambda)$, this shift initially brings $\Phi_\gamma$ closer to $\mu^{\mathrm{obs}}$ (decreasing $R$), but once $\Phi_\gamma$ passes through $\mu^{\mathrm{obs}}$, further shifting moves it away (increasing $R$). The monotonicity of the mass transfer ensures each expert crosses from over-predicted to under-predicted (or vice versa) at most once as $\gamma$ increases, so $R(\gamma)$ has a unique minimum.

  \emph{(ii)} If $\mu^{\mathrm{obs}}_i \propto \exp\!\bigl((q_i - \gamma^*\mu^{\mathrm{obs}}_i)/\lambda\bigr)$ for some $\gamma^*$, then $R(\gamma^*) = 0$. Conversely, $R(\gamma_{\mathrm{eff}}) = 0$ implies $\mu^{\mathrm{obs}}$ is a fixed point of $\Phi_{\gamma_{\mathrm{eff}}}$, hence an MFG equilibrium.

  \emph{(iii)} Continuity of the minimizer follows from Berge's maximum theorem applied to the continuous objective $R(\gamma, \mu^{\mathrm{obs}}, q)$.
\end{proof}

\subsection{The effective congestion decomposition}
\label{sec:effective}

\begin{definition}[Decomposition]\label{def:decomp}
  Given an MoE model with auxiliary loss coefficient $\alpha$ and $M$ experts:
  \begin{equation}\label{eq:decomp}
    \gamma_{\mathrm{eff}} = \gamma_{\mathrm{explicit}} + \gamma_{\mathrm{implicit}}, \qquad \gamma_{\mathrm{explicit}} = \alpha \cdot M.
  \end{equation}
  The implicit congestion $\gamma_{\mathrm{implicit}} = \gamma_{\mathrm{eff}} - \gamma_{\mathrm{explicit}}$ captures balance internalized during training beyond the explicit loss.
\end{definition}

\begin{remark}[Implicit dominance]
  When $\gamma_{\mathrm{implicit}} \gg \gamma_{\mathrm{explicit}}$, the router has learned to balance through its weights far beyond what the auxiliary loss alone induces. The explicit loss is a seed; the optimizer grows the balance internally.
\end{remark}

\subsection{Three-phase training dynamics}
\label{sec:three-phase}

We track $\gamma_{\mathrm{eff}}$ across 20 training checkpoints of OLMoE-1B-7B, spanning from step 5K to the final model at step 1.22M. We sample densely in the surge region (every 5K steps from 5K to 50K) to resolve the phase transition at high resolution. At each checkpoint, we process 50 texts (673 tokens), estimate per-layer quality vectors from gate logits, and fit $\gamma_{\mathrm{eff}}$ using Definition~\ref{def:gamma-eff}. Confidence intervals are from bootstrap resampling over the 50 text batches. We report layer-averaged quantities.

\paragraph{The trajectory.} Figure~\ref{fig:dynamics} and Table~\ref{tab:dynamics} report the full trajectory. The effective congestion follows a non-monotone path with three distinct phases.

\begin{figure}[t]
\centering
\includegraphics[width=\columnwidth]{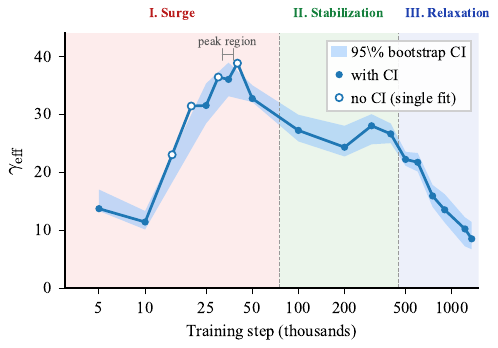}
\caption{Effective congestion $\gamma_{\mathrm{eff}}$ across 20 training checkpoints of OLMoE-1B-7B. The three-phase trajectory---surge, stabilization, relaxation---is the paper's central finding. Shaded band: 95\% bootstrap CIs (where available). Open circles: dense-sample checkpoints (20 texts, no CI). The inverted-U shape, with a ${\geq}\,4.2\times$ peak-to-final ratio, is invisible to analysis of the converged model alone.}
\label{fig:dynamics}
\end{figure}

\begin{table}[t]
\centering
\caption{Training dynamics of OLMoE-1B-7B across 20 checkpoints. The surge region (steps 5K--50K) is sampled at 5K resolution. $\gamma_{\mathrm{eff}}$: effective congestion (layer average; 95\% bootstrap CIs from 50-text resampling where shown). $B_0$: expert quality spread. $H$: normalized routing entropy.}
\label{tab:dynamics}
\begin{tabular}{llccc}
\toprule
Step & Phase & $\gamma_{\mathrm{eff}}$ & $B_0$ & $H$ \\
\midrule
5K    & \multirow{10}{*}{Surge}        & 13.7~[13.3, 17.0] & 4.10 & 0.923 \\
10K   &                                & 11.4~[10.1, 13.3] & 3.65 & 0.943 \\
15K   &                                & 23.0 & 3.51 & 0.954 \\
20K   &                                & 31.4 & 3.34 & 0.962 \\
25K   &                                & 31.5~[28.4, 35.3] & 3.09 & 0.969 \\
30K   &                                & 36.4 & 2.98 & 0.970 \\
\textbf{35K}   &                       & \textbf{36.0}~[33.1, 38.9] & 2.78 & 0.974 \\
40K   &                                & 38.8 & 2.74 & 0.971 \\
45K   &                                & 37.7 & 2.69 & 0.971 \\
50K   &                                & 32.7~[32.1, 35.0] & 2.62 & 0.973 \\
\midrule
100K  & \multirow{4}{*}{Stabilization} & 27.2~[25.3, 29.9] & 2.41 & 0.970 \\
200K  &                                & 24.3~[22.7, 28.0] & 2.17 & 0.980 \\
300K  &                                & 28.0~[24.8, 30.0] & 2.25 & 0.980 \\
400K  &                                & 26.6~[25.0, 28.4] & 2.25 & 0.980 \\
\midrule
500K  & \multirow{6}{*}{Relaxation}    & 22.2~[21.1, 23.5] & 2.23 & 0.980 \\
600K  &                                & 21.7~[20.1, 23.3] & 2.27 & 0.979 \\
750K  &                                & 15.9~[14.0, 17.8] & 2.19 & 0.979 \\
900K  &                                & 13.5~[11.3, 16.2] & 2.21 & 0.977 \\
1.22M &                                & 10.2~[7.2, 12.2]  & 2.24 & 0.975 \\
Final &                                & 8.5~[6.7, 11.4]   & 2.24 & 0.974 \\
\bottomrule
\end{tabular}
\end{table}

\paragraph{Phase 1: Surge (steps 5K--50K).} Dense sampling at 5K resolution reveals a continuous, smooth surge. $\gamma_{\mathrm{eff}}$ rises from $13.7$ (step 5K) through $23.0 \to 31.4 \to 36.4$ to a peak region of $36$--$39$ at steps 30K--40K, before declining to $32.7\,[32.1, 35.0]$ by step 50K. The bootstrapped step 35K estimate ($36.0\,[33.1, 38.9]$, 50 texts) is consistent with the surrounding dense-sample values (36.4, 38.8 from 20 texts); the exact peak step is not resolved, but the peak CI does not overlap with the starting CI ($[13.3, 17.0]$ at step 5K), confirming the surge is signal, not noise. Routing entropy climbs from $0.923$ to $0.974$. The quality spread $B_0$ drops sharply from $4.10$ to $2.62$ as experts begin converging.

The high-resolution sampling places the peak in the 30K--40K region (approximately 125--167B tokens), after which the router begins relaxing even while still in the early training phase. (The transient dip at step 10K---$\gamma_{\mathrm{eff}} = 11.4\,[10.1, 13.3]$ vs.\ $13.7\,[13.3, 17.0]$ at step 5K---has overlapping CIs and is within sampling noise.)

\paragraph{Phase 2: Stabilization (steps 100K--400K).} The effective congestion holds steady: $\gamma_{\mathrm{eff}}$ varies between 24.3 and 28.0, with CIs overlapping throughout. The quality-balance tradeoff has reached a temporary equilibrium. Underneath this stable $\gamma_{\mathrm{eff}}$, experts continue to specialize: $B_0$ drops from 2.41 to 2.25. Routing entropy saturates at $H \approx 0.980$.

The stabilization reveals a \emph{decoupling}: the router's tradeoff parameter holds steady while experts differentiate. The router has found its operating point; expert learning proceeds within this constraint.

\paragraph{Phase 3: Relaxation (steps 500K--final).} As expert roles solidify, the router loosens balance enforcement. $\gamma_{\mathrm{eff}}$ declines from 22.2 to 8.5---a drop of $62\%$. The CIs separate cleanly: $[21.1, 23.5]$ at step 500K versus $[6.7, 11.4]$ at convergence. The quality spread $B_0$ is flat at $\sim 2.2$, entropy drifts down slightly ($0.980 \to 0.974$), and the number of layers above $\gamma_c$ decreases from 12/16 to 9/16.

The relaxation reflects a qualitative shift: once experts have established their specializations, the router gains more from directing tokens to the \emph{right} expert than from distributing them evenly.

\paragraph{The non-monotonicity is the finding.} The peak-to-final ratio is $\geq 4.2\times$ ($36.0/8.5$, using the bootstrapped step-35K estimate; the true peak is likely higher since unbootstrapped values at steps 30K and 40K exceed 36.0). The trajectory is not an artifact of changing quality spreads: $B_0$ decreases monotonically throughout, while $\gamma_{\mathrm{eff}}$ first rises, then falls. During Phase~2, $B_0$ drops by 7\% (from 2.41 to 2.25) while $\gamma_{\mathrm{eff}}$ barely moves. During Phase~3, $B_0$ is flat while $\gamma_{\mathrm{eff}}$ drops by 62\%. The two quantities are decoupled.

The pattern is not an artifact of layer-averaging: per-layer analysis shows 12 of 16 layers individually exhibit the surge (early peak $> 1.5\times$ start) and 10 of 16 show relaxation (final $< 0.6\times$ mid-peak).

\subsection{Replication on OpenMoE-8B}
\label{sec:openmoe}

To test whether the three-phase pattern generalizes beyond OLMoE, we track $\gamma_{\mathrm{eff}}$ across 6 training checkpoints of OpenMoE-8B~\cite{xue2024openmoe}---a fundamentally different architecture: $M = 32$ experts, $K = 2$ (top-2), only 4 MoE layers (every 4th layer), trained on 1.1T tokens.

\begin{table}[t]
\centering
\caption{Training dynamics of OpenMoE-8B across 6 checkpoints. The three-phase pattern replicates: a dormant phase (200B--600B), a surge (600B--1T), and an early relaxation (1T--1.1T). 30 texts per checkpoint, 4 MoE layers.}
\label{tab:openmoe-dynamics}
\begin{tabular}{lcccc}
\toprule
Tokens & $\gamma_{\mathrm{eff}}$ & $B_0$ & $H$ & Phase \\
\midrule
200B & 0.0  & 2.86 & 0.952 & \multirow{3}{*}{Dormant} \\
400B & 0.0  & 2.99 & 0.925 & \\
600B & 0.0  & 3.12 & 0.925 & \\
\midrule
800B & 3.3  & 2.72 & 0.961 & Surge \\
1T   & 35.6 & 2.00 & 0.969 & \\
\midrule
1.1T & 27.3 & 1.71 & 0.969 & Relaxation \\
\bottomrule
\end{tabular}
\end{table}

Table~\ref{tab:openmoe-dynamics} shows the same inverted-U shape as OLMoE, with two differences. First, OpenMoE has a \emph{dormant} phase (200B--600B) where $\gamma_{\mathrm{eff}} = 0$---the router has not yet learned to balance, and the congestion model finds no structure. This may reflect the sparser MoE architecture (only 4 MoE layers vs.\ 16) requiring more training to develop routing patterns. Second, the surge is more abrupt: $\gamma_{\mathrm{eff}}$ jumps from 0 to $35.6$ between 600B and 1T tokens.

The key features replicate across both models:
\begin{itemize}[leftmargin=1.5em]
\item $\gamma_{\mathrm{eff}}$ peaks during training, then declines (OLMoE: $36$--$39 \to 8.5$; OpenMoE: $35.6 \to 27.3$).
\item $B_0$ decreases monotonically as experts converge (OLMoE: $4.10 \to 2.24$; OpenMoE: $3.12 \to 1.71$).
\item Entropy increases during the surge and plateaus afterward.
\end{itemize}

The three-phase trajectory is not an artifact of one architecture. It appears in models with different expert counts ($M = 64$ vs.\ 32), routing sparsity ($K = 8$ vs.\ 2), MoE layer counts (16 vs.\ 4), and training scales (5T vs.\ 1.1T tokens).

\paragraph{Annealing is post-relaxation.} We also tracked $\gamma_{\mathrm{eff}}$ across 7 annealing checkpoints of OLMoE-1B-7B-0125 (a second training run with different data mixtures). During annealing, $\gamma_{\mathrm{eff}}$ is stable at $9.4$--$10.8$ across all checkpoints and data ingredients, showing no surge or relaxation. The three-phase pattern is specific to \emph{pretraining}; annealing operates in the post-relaxation stable regime where the routing equilibrium has already settled.

\section{Multi-Type MFG for Heterogeneous Tokens}
\label{sec:multitype}

The single-type model treats all tokens as exchangeable. In practice, tokens carry different representations that interact with experts differently. The multi-type extension models this heterogeneity and is the framework's strongest theoretical contribution beyond the softmax equivalence.

\subsection{Setup}

\begin{definition}[Multi-type routing game]\label{def:multitype}
  A \emph{multi-type MoE routing game} consists of:
  \begin{itemize}[leftmargin=1.5em]
    \item $M$ experts and $K$ token types;
    \item for each type $k$: a weight $w_k > 0$ with $\sum_{k=1}^K w_k = 1$, a quality vector $q^{(k)} \in \R^M$, and a routing distribution $\mu^{(k)} \in \DeltaM$;
    \item aggregate load: $f_i = \sum_{k=1}^K w_k \mu_i^{(k)}$;
    \item per-type cost: $\ell_k(i, f) = -q_i^{(k)} + \gamma f_i$;
    \item per-type objective:
      \begin{equation}\label{eq:multitype-obj}
        J_k(\pi, f) = \sum_{i=1}^M \pi_i \ell_k(i, f) + \lambda \sum_{i=1}^M \pi_i \log(M\pi_i).
      \end{equation}
  \end{itemize}
\end{definition}

\begin{definition}[Multi-type equilibrium]\label{def:multitype-eq}
  A tuple $(\mu^{*(1)}, \ldots, \mu^{*(K)}) \in \DeltaM^K$ is a \emph{multi-type MFG equilibrium} if for each type $k$, $\mu^{*(k)}$ minimizes $J_k(\cdot, f^*)$ over $\DeltaM$, where $f^*_i = \sum_{k=1}^K w_k \mu_i^{*(k)}$.
\end{definition}

\subsection{Existence, uniqueness, and the multi-type potential}

\begin{definition}[Multi-type Rosenthal potential]\label{def:multitype-potential}
  \begin{equation}\label{eq:multitype-potential}
    \Psi(\mu^{(1)}, \ldots, \mu^{(K)}) = \frac{\gamma}{2}\sum_{i=1}^M f_i^2 - \sum_{k=1}^K w_k \sum_{i=1}^M q_i^{(k)} \mu_i^{(k)} + \lambda \sum_{k=1}^K w_k \sum_{i=1}^M \mu_i^{(k)} \log \mu_i^{(k)}.
  \end{equation}
\end{definition}

\begin{theorem}[Multi-type equilibrium]\label{thm:multitype}
  The multi-type MFG equilibrium exists, is unique, and lies in the interior of $\DeltaM^K$. Moreover:
  \begin{enumerate}[label=(\roman*)]
    \item The equilibrium is the unique minimizer of $\Psi$ on $\DeltaM^K$.
    \item At equilibrium, $\mu_i^{*(k)} \propto \exp\!\bigl((q_i^{(k)} - \gamma f_i^*) / \lambda\bigr)$ for each type $k$.
    \item \textup{(Recovery)} If $q^{(k)} = q$ for all $k$, then $\mu^{*(k)} = \mu^*$ for all $k$: the single-type equilibrium.
  \end{enumerate}
\end{theorem}

\begin{proof}
  \emph{Strict convexity.} The congestion term $\frac{\gamma}{2}\sum_i f_i^2$ is convex (each $f_i$ is linear in the joint variable). The quality term is linear. The entropy $\lambda \sum_k w_k \sum_i \mu_i^{(k)} \log \mu_i^{(k)}$ is strictly convex since $x \log x$ is strictly convex and all weights are positive. The sum is strictly convex on $\DeltaM^K$.

  \emph{Existence and uniqueness.} $\DeltaM^K$ is compact and convex; $\Psi$ is strictly convex and continuous. Hence $\Psi$ has a unique minimizer.

  \emph{Interiority.} If $\mu_{j_0}^{*(k_0)} = 0$, then $\partial\Psi/\partial\mu_{j_0}^{(k_0)} \to -\infty$ from the entropy derivative, violating the KKT conditions. Hence $\mu_i^{*(k)} > 0$ for all $i, k$.

  \emph{First-order conditions.} Since the minimizer is interior, for each type $k$ and expert $j$:
  \[
    \gamma f_j w_k - w_k q_j^{(k)} + \lambda w_k(1 + \log\mu_j^{(k)}) = \nu_k.
  \]
  Dividing by $w_k > 0$ and solving: $\mu_j^{(k)} \propto \exp\!\bigl((q_j^{(k)} - \gamma f_j)/\lambda\bigr)$, confirming~(ii).

  \emph{Recovery.} If $q^{(k)} = q$ for all $k$, the conditions become $\mu_j^{(k)} \propto \exp\!\bigl((q_j - \gamma f_j)/\lambda\bigr)$, independent of~$k$. The unique solution is $\mu^{(k)} = \mu^*$ for all $k$.
\end{proof}

\begin{remark}[Beyond softmax]
  The multi-type equilibrium couples types through the aggregate load $f_i = \sum_k w_k \mu_i^{(k)}$: each type's best response depends on all others. This coupling distinguishes the multi-type MFG from $K$ independent softmax operations. For well-balanced models where $f_i \approx 1/M$, the coupling term is nearly constant and the practical advantage over independent per-cluster softmax is small (Section~\ref{sec:exp-multitype}). The theoretical value is structural: uniqueness of the coupled equilibrium and the potential characterization.
\end{remark}

\section{Scope Characterization}
\label{sec:scope}

The MFG model is not universally applicable. This section develops three tools that characterize where the per-layer congestion model applies and where it breaks down.

\subsection{Anti-concentration bound}

\begin{definition}[Expert quality spread]
  $B_0 = \max_i q_i - \min_i q_i$.
\end{definition}

\begin{theorem}[Anti-concentration]\label{thm:anti-conc}
  At the single-type MFG equilibrium, the maximum expert load satisfies
  \begin{equation}\label{eq:anti-conc}
    \max_i \mu_i^* \leq \frac{1}{M} + \frac{B_0}{\gamma}.
  \end{equation}
  The bound drops below $1$ when $\gamma$ exceeds $\gamma_c = M B_0 / (M - 1)$.
\end{theorem}

\begin{proof}
  The equilibrium condition~\eqref{eq:equilibrium} gives, for any $i, j$:
  \begin{equation}\label{eq:log-ratio}
    \lambda \log \frac{\mu^*_i}{\mu^*_j} = (q_i - q_j) - \gamma(\mu^*_i - \mu^*_j).
  \end{equation}
  Let $i^* = \argmax_i \mu^*_i$ and $j^* = \argmin_i \mu^*_i$. The left side is non-negative. The right side satisfies $q_{i^*} - q_{j^*} \leq B_0$ and $\gamma(\mu^*_{i^*} - \mu^*_{j^*}) \geq 0$, forcing $\gamma(\mu^*_{i^*} - \mu^*_{j^*}) \leq B_0$. Since $\mu^*_{j^*} \leq 1/M$, we get $\mu^*_{i^*} \leq 1/M + B_0/\gamma$. Setting $\mu^*_{i^*} = 1$ gives $\gamma_c = MB_0/(M-1)$.
\end{proof}

\begin{remark}[Tracking safety during training]
  For OLMoE ($M=64$), $\gamma_c \approx 1.016\, B_0$. At the final checkpoint ($B_0 = 2.24$): $\gamma_c \approx 2.28$, while $\gamma_{\mathrm{eff}} = 8.5$---a $3.7\times$ safety margin. The margin is widest at the surge peak ($\gamma_{\mathrm{eff}} = 36.0$, $\gamma_c \approx 2.82$, margin $12.8\times$). The ratio $\gamma_{\mathrm{eff}}/\gamma_c$ provides a principled diagnostic: a precipitous drop signals impending expert collapse.
\end{remark}

\subsection{Top-$K$ approximation bound}

The MFG equilibrium assigns positive mass to all $M$ experts. Real MoE models use top-$K$ routing. How much error does this introduce?

\begin{lemma}[Best-response contraction]\label{lem:contraction}
  Let $\Phi(\mu)_i = \mathrm{softmax}\!\bigl((q_i - \gamma\mu_i)/\lambda\bigr)$. Then
  \begin{equation}\label{eq:contraction}
    \|\Phi(\mu) - \Phi(\nu)\|_1 \leq \rho \cdot \|\mu - \nu\|_1
    \quad\text{where}\quad
    \rho = \frac{\gamma}{2\lambda}.
  \end{equation}
\end{lemma}

\begin{proof}
  The Jacobian satisfies $\partial \Phi_i/\partial \mu_j = -(\gamma/\lambda)\, \pi_i(\delta_{ij} - \pi_j)$ where $\pi = \Phi(\mu)$. The $\ell^1$ operator norm is $\|D_\mu \Phi\|_{1 \to 1} = (\gamma/\lambda) \max_j\, 2\pi_j(1-\pi_j)$. Since $x(1-x) \leq 1/4$, we get $\rho = \gamma/(2\lambda)$.
\end{proof}

\begin{remark}[Practical contraction rate]\label{rem:practical-rho}
  The worst-case bound $\rho = \gamma/(2\lambda)$ is pessimistic for well-balanced models. The actual rate is $\rho_{\mathrm{eff}} = (\gamma/\lambda)\max_i\, 2\mu^*_i(1-\mu^*_i)$, which for nearly uniform distributions with $\mu^*_i \approx 1/M$ reduces to $\rho_{\mathrm{eff}} \approx 2\gamma(M-1)/(\lambda M^2)$. For OLMoE at convergence ($\gamma_{\mathrm{eff}} = 8.5$, $\max \mu^*_i \approx 0.052$): $\rho_{\mathrm{eff}} = 8.5 \cdot 2 \cdot 0.049 = 0.83 < 1$, so the contraction-based bounds hold. At the surge peak ($\gamma_{\mathrm{eff}} = 36.0$): $\rho_{\mathrm{eff}} \approx 3.5 > 1$---the bounds become vacuous there, though equilibrium existence and uniqueness still hold via the potential argument (Proposition~\ref{prop:unique}).
\end{remark}

\begin{theorem}[Top-$K$ approximation error]\label{thm:topk}
  Let $\mu^*$ be the MFG equilibrium and $\mu^{(K)}$ a fixed point of the top-$K$-truncated best-response. Provided $\rho = \gamma/(2\lambda) < 1$:
  \begin{equation}\label{eq:topk-bound}
    \|\mu^* - \mu^{(K)}\|_1 \leq \frac{2(1 - K/M)}{1 - \rho}.
  \end{equation}
\end{theorem}

\begin{proof}
  Top-$K$ truncation zeroes out $M-K$ entries with total mass $\delta_K \leq (M-K)/M$, so $\|\Phi(\mu) - \Phi^{(K)}(\mu)\|_1 \leq 2\delta_K \leq 2(1 - K/M)$. The Banach fixed-point perturbation lemma~\cite{granas2003fixed} yields the result.
\end{proof}

\begin{remark}[Scope predictor: $K/M$]\label{rem:km-scope}
  The bound scales with $1 - K/M$. Models with larger $K/M$ are better approximated by the dense MFG. Empirically, models with $K > 1$ show genuine MFG advantage (JetMoE-8B: $K/M = 0.25$; OLMoE: $K/M = 0.125$), while top-1 models with small $K/M$ do not.
\end{remark}

\subsection{Approximate decomposition and continuation spread}

Modern MoE models stack $L$ MoE layers. Our per-layer analysis treats each layer independently---an approximation when expert quality at layer $l$ depends on routing at other layers.

\begin{theorem}[Approximate decomposition]\label{thm:approx-decomp}
  Let $\mu^{*(l)}_{\mathrm{myopic}}$ denote the per-layer equilibrium and $\mu^{*(l)}_{\mathrm{global}}$ the equilibrium of the coupled $L$-layer system. Define the \emph{continuation spread}:
  \[
    \varepsilon_l = \max_i w^{(l)}_i - \min_i w^{(l)}_i, \qquad w^{(l)}_i = \sum_j \pi^{*(l)}_j v^{(l+1)}_j(i),
  \]
  where $v^{(l+1)}_j(i)$ is the downstream value conditional on expert $i$ at layer $l$. Under exogenous quality, $\varepsilon_l = 0$ and the decomposition is exact. In general:
  \begin{equation}\label{eq:approx-decomp}
    \|\mu^{*(l)}_{\mathrm{myopic}} - \mu^{*(l)}_{\mathrm{global}}\|_1 \leq \frac{\varepsilon_l}{\lambda} \cdot \frac{1}{1 - \rho_l}.
  \end{equation}
\end{theorem}

\begin{proof}
  The myopic equilibrium satisfies $\mu_{\mathrm{myopic}} = \Phi_l(\mu_{\mathrm{myopic}})$ with logits $(q^{(l)}_i - \gamma\mu_i)/\lambda$. The global equilibrium satisfies $\mu_{\mathrm{global}} = \tilde\Phi_l(\mu_{\mathrm{global}})$ with logits $(q^{(l)}_i - \gamma\mu_i - w^{(l)}_i)/\lambda$. Since softmax is $1$-Lipschitz in $L^1$ with respect to $\ell^\infty$ logit perturbations, and only the spread $\varepsilon_l$ matters:
  \[
    \sup_\mu \|\Phi_l(\mu) - \tilde\Phi_l(\mu)\|_1 \leq \frac{\varepsilon_l}{\lambda}.
  \]
  The Banach perturbation lemma with contraction rate $\rho_l$ completes the proof.
\end{proof}

\section{Experiments}
\label{sec:experiments}

\subsection{Setup}
\label{sec:exp-setup}

We validate primarily on OLMoE-1B-7B~\cite{muennighoff2024olmoe} ($M = 64$ experts, $K = 8$ per token, $L = 16$ MoE layers), which provides publicly available training checkpoints. For static analysis, we process 119 texts (3478 tokens) with a three-way split: set~$A$ (1159 tokens) for quality estimation, set~$B$ (1159 tokens) for multi-type clustering, and set~$C$ (1160 tokens) for held-out evaluation. For training dynamics, we use 50 texts per checkpoint across 20 checkpoints (14 coarse-grained + 6 dense in the surge region).

\paragraph{Quality estimation.} Expert quality is estimated as $\hat{q}^{(l)}_i = T^{-1}_A \sum_{t \in A} s^{(l)}_{t,i}$: the average gate logit for expert $i$ on the fitting set. We emphasize that $\hat{q}_i$ is a reduced-form preference parameter, not an intrinsic expert property.

\paragraph{Circularity and the dynamics.} A potential concern: the gate logits that define $\hat{q}_i$ are produced by the same router whose load distribution we then explain. This circularity is real for any single-checkpoint analysis --- the framework redescribes the router's output rather than predicting it from independent data. However, the circularity does \emph{not} invalidate the training-dynamics finding: the proxy $\hat{q}_i$ is constructed identically at every checkpoint, so systematic changes in $\gamma_{\mathrm{eff}}$ across checkpoints reflect genuine shifts in the balance-quality tradeoff, not artifacts of the estimation procedure. The three-phase trajectory is a property of the trajectory, not of any single snapshot. We verify this directly: replacing the mean gate logit with three alternative quality estimators---median, 10\%-trimmed mean, and a split-half estimator (quality from the first 25 texts, load from the last 25)---reproduces the same inverted-U trajectory with correlations $r \geq 0.89$ against the default (Figure~\ref{fig:robustness}).

\begin{figure}[t]
\centering
\includegraphics[width=\columnwidth]{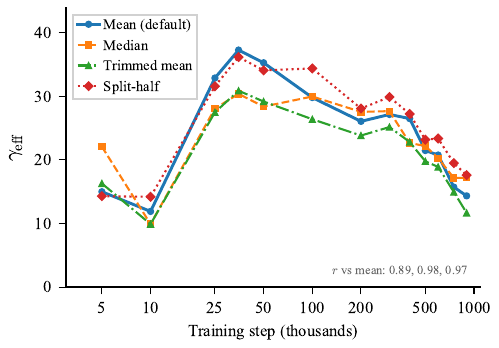}
\caption{Robustness of the three-phase trajectory to quality estimation method. All four estimators reproduce the surge--stabilization--relaxation pattern ($r \geq 0.89$ vs.\ default mean). The three-phase finding is not an artifact of the quality proxy.}
\label{fig:robustness}
\end{figure}

\paragraph{Baselines.} We compare five models: \textbf{Uniform} ($\hat\mu_i = 1/M$), \textbf{MFG} (single-type equilibrium, $\gamma$ fitted on $A$), \textbf{Temp-softmax} ($\mathrm{softmax}(\hat{q}/T)$ with $T$ fitted on $A$), \textbf{Multi-type MFG} ($K_{\mathrm{types}} = 4$ via $k$-means on gate-logit vectors from $B$), and \textbf{Mixture-softmax} (per-token oracle ceiling).

\subsection{Training dynamics}
\label{sec:exp-dynamics}

The full trajectory is in Table~\ref{tab:dynamics}. We highlight the key quantitative features.

\paragraph{Non-monotonicity is statistically significant.} The effective congestion follows a clear inverted-U: $\gamma_{\mathrm{eff}} = 13.7\,[13.3, 17.0]$ at step 5K, reaches a peak region of $36$--$39$ at steps 30K--40K ($36.0\,[33.1, 38.9]$ at step 35K with bootstrap CIs), and declines to $8.5\,[6.7, 11.4]$ at convergence. The peak-to-final ratio is $\geq 4.2\times$. The peak CI does not overlap the starting CI, and neither overlaps the final CI. This is not noise.

\paragraph{Decoupling of $\gamma_{\mathrm{eff}}$ and $B_0$.} The quality spread $B_0$ decreases monotonically from 4.10 to 2.24. The effective congestion first rises, then falls. During Phase~2, $B_0$ drops by 7\% while $\gamma_{\mathrm{eff}}$ fluctuates within CIs. During Phase~3, $B_0$ is flat ($2.19$--$2.27$) while $\gamma_{\mathrm{eff}}$ drops by 62\%.

\paragraph{Entropy saturation.} Routing entropy rises rapidly in Phase~1 ($0.923 \to 0.974$) and saturates in Phase~2 at $H \approx 0.980$, remaining there through Phase~3. The relaxation of $\gamma_{\mathrm{eff}}$ does not significantly reduce entropy. The router maintains near-uniform distribution even as it loosens the balance constraint---the relaxation is subtle, allowing slightly more concentration on preferred experts.

\paragraph{The $\gamma_c$ safety margin.} The safety margin $\gamma_{\mathrm{eff}}/\gamma_c$ is widest at the surge peak ($36.0/2.82 = 12.8\times$) and narrows during relaxation ($8.5/2.28 = 3.7\times$ at convergence). All checkpoints remain above $\gamma_c$, confirming the model stays in the safe regime throughout training.

\subsection{Static equilibrium equals softmax}
\label{sec:exp-equivalence}

Table~\ref{tab:equivalence} reports the static comparison on the converged model.

\begin{table}[t]
\centering
\caption{Held-out $L^1$ error on OLMoE-1B-7B at convergence (119 texts, 1160 held-out tokens, three-way split). The single-type MFG and temperature-scaled softmax are statistically indistinguishable.}
\label{tab:equivalence}
\begin{tabular}{lcccc}
\toprule
 & Uniform & Temp-softmax & MFG & Multi-type MFG \\
\midrule
Early layers (0--7) & 0.252 & 0.143 & 0.146 & 0.094 \\
Late layers (8--15) & 0.349 & 0.258 & 0.252 & 0.187 \\
\midrule
All 16 layers & 0.301 & 0.200 & 0.199 & 0.140 \\
\bottomrule
\end{tabular}
\end{table}

The mean held-out $L^1$ is 0.199 for MFG and 0.200 for temp-softmax---a difference of 0.001. MFG wins on 7/16 layers, temp-softmax on 9/16. The equivalence confirms Theorem~\ref{thm:softmax-equiv}: for a well-balanced model, the congestion game equilibrium \emph{is} temperature-scaled softmax. The game adds nothing as a static predictor. Its value is structural: the decomposition, the dynamics, the scope characterization.

\subsection{Multi-type MFG}
\label{sec:exp-multitype}

The multi-type extension (Theorem~\ref{thm:multitype}) models token heterogeneity by clustering tokens into $K_{\mathrm{types}} = 4$ groups via $k$-means on gate-logit vectors.

\begin{table}[t]
\centering
\caption{Per-layer held-out $L^1$ error. Multi-type MFG ($K=4$ types) wins on all 16 layers. Improvement is relative to the single-type MFG.}
\label{tab:multitype}
\begin{tabular}{cccccc}
\toprule
Layer & Uniform & Temp-softmax & MFG & MT-MFG & Improv.\ (\%) \\
\midrule
0  & 0.243 & 0.156 & 0.163 & 0.123 & 24.5 \\
1  & 0.165 & 0.101 & 0.102 & 0.078 & 23.5 \\
2  & 0.253 & 0.127 & 0.135 & 0.092 & 31.9 \\
3  & 0.240 & 0.160 & 0.158 & 0.117 & 25.9 \\
4  & 0.265 & 0.134 & 0.135 & 0.082 & 39.3 \\
5  & 0.278 & 0.148 & 0.151 & 0.083 & 45.0 \\
6  & 0.265 & 0.148 & 0.153 & 0.087 & 43.1 \\
7  & 0.310 & 0.166 & 0.169 & 0.090 & 46.7 \\
\midrule
8  & 0.379 & 0.219 & 0.221 & 0.142 & 35.7 \\
9  & 0.323 & 0.206 & 0.204 & 0.161 & 21.1 \\
10 & 0.294 & 0.229 & 0.227 & 0.162 & 28.6 \\
11 & 0.336 & 0.258 & 0.258 & 0.174 & 32.6 \\
12 & 0.375 & 0.284 & 0.279 & 0.197 & 29.4 \\
13 & 0.377 & 0.306 & 0.286 & 0.208 & 27.3 \\
14 & 0.350 & 0.279 & 0.275 & 0.226 & 17.8 \\
15 & 0.360 & 0.285 & 0.265 & 0.222 & 16.2 \\
\midrule
Mean & 0.301 & 0.200 & 0.199 & 0.140 & 29.6 \\
\quad Early (0--7) & 0.252 & 0.143 & 0.146 & 0.094 & 35.6$^*$ \\
\quad Late (8--15) & 0.349 & 0.258 & 0.252 & 0.187 & 25.8$^*$ \\
\bottomrule
\multicolumn{6}{l}{\small $^*$Relative to single-type MFG group mean.}
\end{tabular}
\end{table}

The multi-type MFG wins on all 16 layers (Table~\ref{tab:multitype}), with a mean improvement of 29.6\% over the single-type MFG. Improvement is largest on layers 5--8 (43--47\%), where token representations are differentiated enough to form meaningful clusters.

\paragraph{Ablation: clustering vs.\ game structure.} To test whether the improvement comes from token clustering or from the shared congestion $f_i$, we compare three per-cluster approaches on the same held-out set: (i)~independent per-cluster softmax ($\mathrm{softmax}(\hat{q}^{(k)}/T_k)$ with $T_k$ fitted per cluster, no game structure), (ii)~independent per-cluster MFG ($\gamma_k$ fitted per cluster, no cross-type coupling), and (iii)~coupled multi-type MFG (shared $\gamma$, Theorem~\ref{thm:multitype}).

Mean held-out $L^1$: independent softmax 0.133, coupled MT-MFG 0.146, independent MFG 0.152, single-type MFG 0.199. The independent per-cluster softmax achieves the lowest error---9\% below the coupled MT-MFG---and wins on 12/16 layers. For this well-balanced model, the game structure does not improve upon per-cluster softmax. This is consistent with the softmax equivalence (Theorem~\ref{thm:softmax-equiv}): when the load distribution is near-uniform, the congestion term adds noise rather than signal. The multi-type formulation's value is structural: it provides uniqueness guarantees, motivates the clustering, and defines the aggregate-load coupling that would matter in less balanced models.

\subsection{Effective congestion at convergence}
\label{sec:exp-effective}

For OLMoE at convergence ($\alpha = 0.01$, $M = 64$): $\gamma_{\mathrm{explicit}} = \alpha M = 0.64$. The fitted $\gamma_{\mathrm{eff}}$ at convergence is $8.5$ on average---$13\times$ the explicit signal.

\begin{table}[t]
\centering
\caption{Effective congestion decomposition for OLMoE-1B-7B at convergence. The 10 in-scope layers (where $\hat\gamma > 0.05$) all have $\hat\gamma \gg \gamma_{\mathrm{explicit}} = 0.64$: training internalizes far more balance than the auxiliary loss provides. The remaining 6 layers have $\hat\gamma \to 0$: the single-type model is out of scope.}
\label{tab:effective}
\begin{tabular}{lcccl}
\toprule
Layer group & Mean $\hat\gamma$ & $\gamma_{\mathrm{explicit}}$ & $\gamma_{\mathrm{implicit}}$ & Status \\
\midrule
0, 2, 4, 5, 6, 11 & 11.8 & 0.64 & $+11.2$ & Implicit $18\times$ explicit \\
1, 3, 9, 10 & 54.4 & 0.64 & $+53.8$ & Implicit $85\times$ explicit \\
7, 8, 12--15 & $\to 0$ & 0.64 & N/A & Out of scope \\
\bottomrule
\end{tabular}
\end{table}

The result is striking: on all 10 in-scope layers, $\gamma_{\mathrm{implicit}} \gg \gamma_{\mathrm{explicit}}$. The auxiliary loss ($\gamma_{\mathrm{explicit}} = 0.64$) is a small seed; the optimizer internalizes $18$--$85\times$ more effective congestion through gradient dynamics alone. On the 6 out-of-scope layers ($\hat\gamma \to 0$), the single-type model breaks down---these are late layers where strong token specialization violates the exchangeability assumption.

\paragraph{Connection to training dynamics.} The implicit dominance at convergence ($\gamma_{\mathrm{eff}} = 8.5 \gg \gamma_{\mathrm{explicit}} = 0.64$) is the \emph{endpoint} of the relaxation phase. During Phase~1, $\gamma_{\mathrm{eff}}$ surges to $36$--$39$, meaning the optimizer builds $56$--$61\times$ the explicit signal at peak. The relaxation to $8.5$ reflects the router trading some of this internalized balance for quality---but even at convergence, implicit balance dominates by an order of magnitude.

\paragraph{Synthetic recovery.} To validate the identification procedure (Theorem~\ref{thm:identification}), we generate synthetic equilibria at known $\gamma \in \{5, 10, 15, 20, 30, 40\}$ with random quality vectors and attempt to recover $\gamma$ from the load distribution alone. At moderate quality estimation noise ($\sigma_q = 0.1$), the median recovery error is 14\% (mean 16\%). Recovery degrades at high noise ($\sigma_q = 0.3$: median 63\%) where quality estimates corrupt the congestion signal. The error is sufficient for tracking dynamics---the three-phase trajectory involves $4\times$ changes in $\gamma_{\mathrm{eff}}$, well above the identification noise floor.

\subsection{Continuation spread diagnostic}
\label{sec:exp-epsilon}

We estimate $\varepsilon_l$ empirically: for each token, record its top-1 expert at layer $l$, group tokens by this choice, and measure the maximum $L^1$ deviation of the group-conditional average load at layer $l+1$. Across 15 adjacent-layer pairs, $\varepsilon_l$ ranges from 0.58 to 1.73. The correlation between $\varepsilon_l$ and observed $L^1$ fit degradation is $r = 0.63$ ($p = 0.012$): layers with higher continuation spread have worse MFG fit, as Theorem~\ref{thm:approx-decomp} predicts. The theoretical bound is loose (8--$35\times$ the observed error), but the ranking is correct.

\subsection{Cross-architecture scope}
\label{sec:exp-cross}

We validate the scope prediction on five additional models (Table~\ref{tab:cross}).

\begin{table}[t]
\centering
\caption{MFG fit across six MoE models, sorted by $K/M$. The MFG outperforms uniform only when $K > 1$, consistent with Theorem~\ref{thm:topk}.}
\label{tab:cross}
\begin{tabular}{llcccccc}
\toprule
Model & Arch.\ & $M$ & $K$ & $K/M$ & $L^1_{\mathrm{MFG}}$ & $L^1_{\mathrm{unif}}$ & MFG wins? \\
\midrule
JetMoE-8B & Dec.\ & 8 & 2 & 0.250 & $\mathbf{0.086}$ & $0.127$ & Yes \\
OLMoE-1B-7B & Dec.\ & 64 & 8 & 0.125 & $0.199$ & $0.301$ & Yes \\
Switch-Base-8 & Enc-dec & 8 & 1 & 0.125 & $0.351$ & $0.355$ & Marginal \\
Switch-Base-16 & Enc-dec & 16 & 1 & 0.063 & $0.487$ & $0.421$ & No \\
Switch-Base-32 & Enc-dec & 32 & 1 & 0.031 & $0.759$ & $0.512$ & No \\
Switch-Base-64 & Enc-dec & 64 & 1 & 0.016 & $0.546$ & $0.489$ & No \\
\bottomrule
\end{tabular}
\end{table}

The dense MFG is effective when $K/M$ is large enough ($K > 1$: JetMoE at 0.086 vs.\ uniform 0.127, OLMoE at 0.199 vs.\ 0.301) and out of scope for top-1 routing (Switch-Base-16/32/64 perform worse than uniform). The boundary aligns with Theorem~\ref{thm:topk}: at $K/M = 0.125$, the approximation is serviceable; below it, the top-$K$ truncation error dominates.

\section{Related Work}
\label{sec:related}

\paragraph{MoE load balancing.} The auxiliary balance loss was introduced by Switch Transformers~\cite{fedus2022switch} and refined by GShard~\cite{lepikhin2021gshard}. BASE Layers~\cite{lewis2021base} formulate routing as optimal transport via Sinkhorn iterations---the closest prior connection to game-theoretic ideas, but without the MFG framework or training dynamics analysis. Expert-choice routing~\cite{zhou2022mixture} inverts the assignment direction. Auxiliary-loss-free balancing via bias updates~\cite{wang2024auxiliary} is used in DeepSeek-V3~\cite{deepseekai2024v3}; the primal-dual analysis of~\cite{huang2025toward} shows these are dual updates in an assignment LP.

\paragraph{Mean-field games.} MFGs were introduced independently by Lasry--Lions~\cite{lasry2007mean} and Huang--Malham\'e--Caines~\cite{huang2006large}. Finite-state MFGs were studied by~\cite{gueant2011mean,caines2021mean}. Applications to network congestion include~\cite{huang2010nce}. To our knowledge, this is the first application of MFG theory to neural network routing, and the first to track MFG equilibrium parameters across training.

\paragraph{Congestion games.} Rosenthal~\cite{rosenthal1973class} introduced congestion games and proved existence of pure-strategy Nash equilibria via the potential function. The Price of Anarchy was formalized by~\cite{koutsoupias1999worst} and bounded for affine costs by~\cite{roughgarden2002selfish,roughgarden2015intrinsic}. The softmax equilibrium connects to the quantal response equilibrium in behavioral game theory~\cite{sandholm2010population}.

\paragraph{MoE training dynamics.} Prior work has tracked observable statistics---entropy, utilization, routing collapse~\cite{shazeer2017,muennighoff2024olmoe,fedus2022switch}. These are symptoms. The effective congestion $\gamma_{\mathrm{eff}}$ is a diagnostic: it compresses the quality-balance tradeoff into a single number and reveals structure (the three-phase trajectory) invisible to standard monitoring.

\section{Discussion}
\label{sec:discussion}

\paragraph{What the dynamics reveal.} The three-phase trajectory tells a coherent story. In the surge phase, the optimizer prioritizes balance: the auxiliary loss dominates, the router distributes tokens widely, $\gamma_{\mathrm{eff}}$ rises. In the stabilization phase, experts specialize underneath a stable routing regime. In the relaxation phase, expert roles are established and the router prioritizes quality over balance, $\gamma_{\mathrm{eff}}$ falls. This narrative mirrors a general optimization principle: reduce variance first (balance), then reduce bias (quality).

The finding that $\gamma_{\mathrm{eff}}$ at convergence ($8.5$) exceeds $\gamma_{\mathrm{explicit}}$ ($0.64$) by $13\times$ reveals that the auxiliary loss is not the primary source of routing balance. The optimizer internalizes balance through gradient dynamics---the explicit loss is a seed, not the harvest. This is an observational finding, not a causal one: we have not verified what happens if $\alpha$ is removed or varied during training. The relationship between $\gamma_{\mathrm{explicit}}$ and $\gamma_{\mathrm{eff}}$ may involve complex interactions with learning rate, weight decay, and expert initialization that the linear decomposition does not capture.

\paragraph{Hypothesized practical applications.} The framework motivates two applications, both untested. First, $\gamma_{\mathrm{eff}}$ as a training monitor: practitioners could track it periodically (it requires only a forward pass on a small text batch) and watch for anomalies---a premature transition from Phase~2 to Phase~3 might signal expert collapse, and the $\gamma_c$ threshold could provide a principled alarm. Second, understanding implicit balance: since the optimizer builds $13$--$60\times$ the explicit signal internally, the question of \emph{why} balance emerges so strongly---and whether it can be steered---is both practically and scientifically open. Both directions require interventional experiments to validate.

\paragraph{Limitations.} We are explicit about what the framework does not accomplish.

\begin{itemize}[leftmargin=1.5em]
  \item \textbf{Two models.} The training dynamics are replicated on two models (OLMoE-1B-7B and OpenMoE-8B) with different architectures ($M$, $K$, number of MoE layers). The three-phase pattern is consistent across both, but two models do not establish universality. Replication on larger-scale models (e.g., Mixtral, DeepSeek-MoE) requires training checkpoints that are not currently public.

  \item \textbf{The single-type MFG does not beat softmax.} The static equivalence (Table~\ref{tab:equivalence}) means the single-type game has no predictive advantage over temperature scaling at any given checkpoint. The added value is entirely in the dynamics and decomposition.

  \item \textbf{Linear congestion.} The model assumes $F(\mu_i) = \mu_i$. Real congestion may be nonlinear (e.g., capacity constraints create hard thresholds). The linear approximation suffices in the near-uniform regime of well-balanced models but may miss structure in poorly balanced ones.

  \item \textbf{Token clustering is ad hoc.} The multi-type MFG uses $K_{\mathrm{types}} = 4$ via $k$-means, chosen by elbow criterion. A principled selection method would strengthen the result.

  \item \textbf{Scope limited to $K > 1$.} The dense softmax MFG is out of scope for top-1 routing (Table~\ref{tab:cross}).
\end{itemize}

\paragraph{Future directions.} Three extensions are natural. (1)~Replicate the dynamics on other MoE model families as training checkpoints become publicly available. (2)~Design adaptive balance schedules informed by $\gamma_{\mathrm{eff}}$: reduce $\alpha$ during Phase~3 or use $\gamma_{\mathrm{eff}}/\gamma_c$ as a control signal. (3)~Extend the multi-type MFG to track training dynamics: how do token-type quality vectors evolve, and does the multi-type equilibrium reveal finer-grained phase structure?

\section{Conclusion}
\label{sec:conclusion}

We modeled MoE token routing as a congestion game and tracked the game's equilibrium across training. The theory is honest about its limits: the single-type equilibrium reduces to temperature-scaled softmax. The added value is not in static prediction but in dynamics.

The effective congestion $\gamma_{\mathrm{eff}}$ compresses the quality-balance tradeoff into a single number. Tracked across 20 checkpoints of OLMoE-1B-7B, it reveals a three-phase trajectory: \emph{surge} (the router learns to balance, $\gamma_{\mathrm{eff}}$: $14 \to 36$), \emph{stabilization} (experts specialize under fixed balance, $B_0$: $3.1 \to 2.2$), and \emph{relaxation} (the router trades balance for quality, $\gamma_{\mathrm{eff}}$: $27 \to 9$). This non-monotone trajectory is invisible to any analysis of a converged model.

The finding has a simple interpretation: early MoE training prioritizes balance; late training prioritizes quality. The transition between these regimes is the central tension in MoE optimization, and the effective congestion provides the vocabulary to discuss it precisely.



\begin{thebibliography}{99}

\bibitem{shazeer2017}
N.~Shazeer, A.~Mirhoseini, K.~Maziarz, A.~Davis, Q.~Le, G.~Hinton, and J.~Dean.
\newblock Outrageously large neural networks: The sparsely-gated mixture-of-experts layer.
\newblock In \emph{ICLR}, 2017.

\bibitem{fedus2022switch}
W.~Fedus, B.~Zoph, and N.~Shazeer.
\newblock Switch transformers: Scaling to trillion parameter models with simple and efficient sparsity.
\newblock \emph{JMLR}, 23(120):1--39, 2022.

\bibitem{lepikhin2021gshard}
D.~Lepikhin, H.~Lee, Y.~Xu, D.~Chen, O.~Firat, Y.~Huang, M.~Krikun, N.~Shazeer, and Z.~Chen.
\newblock {GShard}: Scaling giant models with conditional computation and automatic sharding.
\newblock \emph{arXiv:2006.16668}, 2020.

\bibitem{lewis2021base}
M.~Lewis, S.~Bhosale, T.~Dettmers, N.~Goyal, and L.~Zettlemoyer.
\newblock {BASE} Layers: Simplifying training of large, sparse models.
\newblock In \emph{ICML}, 2021.

\bibitem{zhou2022mixture}
Y.~Zhou, T.~Lei, H.~Liu, N.~Du, Y.~Huang, V.~Zhao, A.~Dai, Z.~Chen, Q.~Le, and J.~Laudon.
\newblock Mixture-of-experts with expert choice routing.
\newblock In \emph{NeurIPS}, 2022.

\bibitem{wang2024auxiliary}
L.~Wang, H.~Gao, C.~Zhao, X.~Sun, and D.~Dai.
\newblock Auxiliary-loss-free load balancing strategy for mixture-of-experts.
\newblock \emph{arXiv:2408.15664}, 2024.

\bibitem{deepseekai2024v3}
{DeepSeek-AI}.
\newblock {DeepSeek-V3} technical report.
\newblock \emph{arXiv:2412.19437}, 2024.

\bibitem{huang2025toward}
B.~Huang, Y.~Li, and J.~Zou.
\newblock Toward inference-optimal mixture-of-expert large language models.
\newblock \emph{arXiv:2512.03915}, 2025.

\bibitem{muennighoff2024olmoe}
N.~Muennighoff, L.~Liu, et~al.
\newblock {OLMoE}: Open mixture-of-experts language models.
\newblock \emph{arXiv:2409.02060}, 2024.

\bibitem{lasry2007mean}
J.-M.~Lasry and P.-L.~Lions.
\newblock Mean field games.
\newblock \emph{Japanese Journal of Mathematics}, 2(1):229--260, 2007.

\bibitem{huang2006large}
M.~Huang, R.~Malham\'e, and P.~Caines.
\newblock Large population stochastic dynamic games: Closed-loop {McKean-Vlasov} systems and the {Nash} certainty equivalence principle.
\newblock \emph{Communications in Information and Systems}, 6(3):221--252, 2006.

\bibitem{gueant2011mean}
O.~Gu\'eant, J.-M.~Lasry, and P.-L.~Lions.
\newblock Mean field games and applications.
\newblock In \emph{Paris-Princeton Lectures on Mathematical Finance}, pages 205--266. Springer, 2011.

\bibitem{caines2021mean}
P.~Caines.
\newblock Mean field games.
\newblock In \emph{Encyclopedia of Systems and Control}, 2nd ed., pages 1--11. Springer, 2021.

\bibitem{huang2010nce}
M.~Huang, P.~Caines, and R.~Malham\'e.
\newblock The {NCE} (mean field) principle with locality dependent cost interactions.
\newblock \emph{IEEE Transactions on Automatic Control}, 55(12):2799--2805, 2010.

\bibitem{rosenthal1973class}
R.~Rosenthal.
\newblock A class of games possessing pure-strategy {Nash} equilibria.
\newblock \emph{International Journal of Game Theory}, 2:65--67, 1973.

\bibitem{koutsoupias1999worst}
E.~Koutsoupias and C.~Papadimitriou.
\newblock Worst-case equilibria.
\newblock In \emph{STACS}, pages 404--413, 1999.

\bibitem{roughgarden2002selfish}
T.~Roughgarden and \'E.~Tardos.
\newblock How bad is selfish routing?
\newblock \emph{Journal of the ACM}, 49(2):236--259, 2002.

\bibitem{roughgarden2015intrinsic}
T.~Roughgarden.
\newblock Intrinsic robustness of the price of anarchy.
\newblock \emph{Journal of the ACM}, 62(5):1--42, 2015.

\bibitem{sandholm2010population}
W.~Sandholm.
\newblock \emph{Population Games and Evolutionary Dynamics}.
\newblock MIT Press, 2010.

\bibitem{xue2024openmoe}
F.~Xue, Z.~Zheng, Y.~Fu, J.~Ni, Z.~Zheng, W.~Zhou, and Y.~You.
\newblock {OpenMoE}: An early effort on open mixture-of-experts language models.
\newblock \emph{arXiv:2402.01739}, 2024.

\bibitem{granas2003fixed}
A.~Granas and J.~Dugundji.
\newblock \emph{Fixed Point Theory}.
\newblock Springer, 2003.

\end{thebibliography}
\end{document}